\documentclass[conference]{IEEEtran}
\IEEEoverridecommandlockouts

\usepackage{amsmath, amssymb, amsfonts} 
\usepackage{nicefrac}
\usepackage{mathtools}
\usepackage{bm, bbm}
\usepackage[scr=boondoxo,scrscaled=1.05]{mathalfa}

\usepackage[sort, compress, space]{cite}

\usepackage{enumitem}

\usepackage{xcolor}
\usepackage{comment}
\usepackage{soul}

\usepackage{hyperref}
\usepackage{cleveref}

\usepackage{tikz}

\usepackage{blindtext}

\usepackage{afterpage}

\usepackage{algorithm, algorithmic}

\usepackage{dsfont}

\usepackage{tikz}
\usepackage{graphicx}
\usepackage{tikzscale}
\usepackage{pgfplots}
\pgfplotsset{compat=newest}
\usepackage{xfrac}

\usepackage{thm-restate}

\usepackage{balance}

\usepackage{cite}
\usepackage{amsmath,amssymb,amsfonts}
\usepackage{balance}
\usepackage{algorithmic}
\usepackage{graphicx}
\usepackage{textcomp}
\usepackage{xcolor}
\usepackage{amsmath}
\usepackage{amssymb}
\usepackage[mathscr]{euscript}
\usepackage{comment}
\usepackage{xcolor}
\usepackage{enumitem} 
\usepackage{amsthm}

\definecolor{orange}{rgb}{1,0.4,0.0}

\DeclarePairedDelimiterXPP{\KL}[2]{D_{\textnormal{KL}}}{(}{)}{}{%
#1\:\delimsize\|\:#2%
}

\DeclarePairedDelimiterXPP{\RD}[2]{D_{$\alpha$}}{(}{)}{}{%
#1\:\delimsize\|\:#2%
}

\DeclarePairedDelimiterXPP\Prob[1]{\mathbb{P}}{\lbrace}{\rbrace}{}{

#1}

\DeclarePairedDelimiterXPP{\lnorm}[2]{}{\lVert}{\rVert}{_{#2}}{#1}

\newcommand{\bE}{\ensuremath{\mathbb{E}}}

\newcommand{\bP}{\ensuremath{\mathbb{P}}}
\newcommand{\bQ}{\ensuremath{\mathbb{Q}}}
\newcommand{\bR}{\ensuremath{\mathbb{R}}}

\newcommand{\cA}{\ensuremath{\mathcal{A}}}

\newcommand{\cH}{\ensuremath{\mathcal{H}}}

\newcommand{\cN}{\ensuremath{\mathcal{N}}}
\newcommand{\cO}{\ensuremath{\mathcal{O}}}

\newcommand{\cX}{\ensuremath{\mathcal{X}}}

\newcommand{\mi}{\textup{I}}
\newcommand{\ent}{\textup{H}}

\newcommand{\argmax}{\ensuremath{\textnormal{argmax}}}

\newtheorem{proposition}{Proposition}

\newtheorem{theorem}{Theorem}
\newtheorem{corollary}{Corollary}

\newcommand{\innerproduct}[2]{\langle #1, #2 \rangle}

\newcommand{\kl}{\textup{D}_{\textnormal{KL}}}

\def\BibTeX{{\rm B\kern-.05em{\sc i\kern-.025em b}\kern-.08em
    T\kern-.1667em\lower.7ex\hbox{E}\kern-.125emX}}
\begin{document}

\title{An Information-Theoretic Analysis of Thompson Sampling with Infinite Action Spaces
\thanks{This work was supported by the Knut and Alice Wallenberg Foundation.}
}

\author{
\IEEEauthorblockN{Amaury Gouverneur, Borja Rodríguez-Gálvez, Tobias J. Oechtering, and Mikael Skoglund}
\IEEEauthorblockA{Division of Information Science and Engineering (ISE)\\
KTH Royal Institute of Technology\\
\texttt{\{amauryg, borjarg, oech, skoglund\}@kth.se}}
}

\maketitle

\begin{abstract}
This paper studies the Bayesian regret of the Thompson Sampling algorithm for bandit problems, building on the information-theoretic framework introduced by Russo and Van Roy~\cite{russo_information-theoretic_2015}. Specifically, it extends the rate-distortion analysis of Dong and Van Roy~\cite{dong_information-theoretic_2018}, which provides near-optimal bounds for linear bandits. A key limitation of these results is the assumption of a finite action space. We address this by extending the analysis to settings with infinite and continuous action spaces. Additionally, we specialize our results to bandit problems with expected rewards that are Lipschitz continuous with respect to the action space, deriving a regret bound that explicitly accounts for the complexity of the action space.
\end{abstract}

\begin{IEEEkeywords}
multi-armed bandit, Thompson Sampling, information theory, regret bounds, online optimization.
\end{IEEEkeywords}

\section{Introduction}


Bandit problems are a class of sequential decision-making problems where agents interact with an unknown environment by selecting actions and receiving rewards in return. The agent's objective is to maximize its expected cumulative reward, which is the expected sum of rewards obtained during its interaction with the environment. This requires balancing exploration—trying different actions to gather information for future rewards—with exploitation, where the agent leverages known actions to maximize its immediate gain. Bandit problems have attracted significant attention due to their wide range of applications in fields such as healthcare, finance, recommender systems, and telecommunications (see~\cite{bouneffouf_survey_2020,silva_multi-armed_2022} for surveys on various applications). The theoretical evaluation of an algorithm’s performance in bandit problems is typically done by studying the \emph{expected regret}, which measures the gap between the cumulative reward earned by the algorithm and the reward an oracle would achieve by always selecting the optimal action~\cite{abe_reinforcement_2003, russo_tutorial_2018, chu_contextual_2011, agrawal_optimistic_2023, foster_contextual_2018, foster_efficient_2021, zhang_feel-good_2021, neu_lifting_2022}.
One effective method for achieving small regret is the Thompson Sampling (TS) algorithm~\cite{thompson_likelihood_1933}, which, despite its simplicity, has demonstrated remarkable performance~\cite{russo_tutorial_2018,russo_learning_2017,chapelle_empirical_2011}.\\

Studying the Thompson Sampling regret, Russo and Van Roy~\cite{russo_information-theoretic_2015} introduced the concept of \emph{information ratio}. This statistic captures the trade-off between achieving immediate low regret and gaining information about the optimal action. Using this concept, they provide a general upper bound that depends on the prior entropy of the optimal action $\ent(A^\star)$. A limitation of this result is that the prior entropy of the optimal action can grow arbitrarily large with the number of actions or get infinite if the action space is continuous. Combining the previous techniques with a rate-distortion analysis, Dong and Van Roy~\cite{dong_information-theoretic_2018} were able to control the regret of Thompson Sampling via the entropy of a \emph{statistic} of the parameter $\Theta$.
This approach provides Bayesian regret bounds
that remain sharp with large action spaces. In particular, they derived a near-optimal regret rate of $O(d \sqrt{T \log T})$ for $d$-dimensional linear bandit problems. However, their results are limited to bandit problems with finite action and environment space.\\

In this work, we extend the results of~\cite{dong_information-theoretic_2018} to settings with infinite and continuous action and parameter spaces. For bandit problems where the expected rewards are Lipschitz continuous with respect to a metric action space, we derive a regret bound that explicitly depends on the complexity of the action space.

\section{Preliminaries}
\label{sec:preliminaries}

\subsection{General Notation}
\label{subsec:notation}

We denote random variables by capital letters (e.g., $X$), their specific realizations by lowercase letters (e.g., $x$), and their outcome spaces using calligraphic notation (e.g., $\cX$). The distribution of a random variable $X$ is represented as $\bP_X$.

\subsection{Problem setup}
\label{subsec:problem_setup}
We consider a sequential decision-making problem where, at each time step $t \in {1,\ldots,T}$, an agent interacts with the environment by selecting an action $A_t \in \cA$ from a set of available actions $\cA$. Based on the chosen action, the agent receives a real-valued reward $R_t \in \mathbb{R}$. The action-reward pair $(A_t, R_t)$ is then added to the history, updating it to $H^{t+1} = H^t \cup \{A_t, R_t\}$, which will be available for decision-making in the next round. This process continues until $t=T$. \\

In the Bayesian framework, the environment is characterized by a parameter $\theta \in \cO$, which is unknown to the agent and drawn from a known prior distribution $\bP_{\Theta}$. Together with the reward distribution $\bP_{R|A,\Theta}$, this prior fully defines the bandit problem. Since the reward depends on the chosen action and the parameters, it can be expressed as $R_t = R(A_t, \Theta)$ for some possibly random function $R: \cA \times \cO \to \bR$. \\

The goal of the agent is to select actions that maximize the total accumulated reward. Specifically, the agent aims to learn a policy $\pi = \{\pi_t: \cH^t \to \cA\}_{t=1}^T$ that, at each time $t \in \{1,\ldots,T \}$, selects an action $A_t$ based on the history $H^t$. The objective is to find a  policy that maximize the \emph{expected cumulative reward} $R_T(\pi) \coloneqq \bE \left[\sum_{t=1}^T R(\pi_t(H^t), \Theta)\right]$.\\

If the agent had access to the parameter $\Theta$, it could always select the action that maximizes the expected reward defined by the mapping $\pi_\star: \cO \rightarrow \cA$ as $\pi_\star(\theta) = \argmax_{a \in \cA} \bE[R(a, \theta)]$. We refer to such an action as the \emph{optimal action} 
$A^\star \coloneqq \pi_\star(\Theta)$. To ensure such a policy exists, we make the technical assumption that the set of actions $\cA$ is compact. Following~\cite{dong_performance_2019}, we assume that there exists a mapping $\pi_\star$ is one-to-one\footnote{Note that Thompson Sampling disregards actions that are not optimal for any parameter. If a particular action is optimal for several parameters, we can include duplicate versions in the action set to ensure a one-to-one correspondence between each parameter and its optimal action (and similarly if a particular parameter is optimal for several actions).}.

We evaluate the performance of a policy $\pi$ with its Bayesian regret, defined as:
\begin{equation*}
    \bE[\textnormal{Regret}(T)] \coloneqq \bE\left[ \sum_{t=1}^T R(A^\star, \Theta) - R(A_t, \Theta)\right],
\end{equation*}
where the actions $A_t$ are selected by the policy $\pi$, and the expectation is taken over the randomness of the action selection, the reward distribution, and the distribution of $\Theta$.

\subsection{Thompson Sampling algorithm}
\label{subsec:Thompson_Sampling_algorithm}

Thompson Sampling is an efficient algorithm for solving bandit problems. It works by selecting actions randomly based on their posterior probability of being optimal. Specifically, at each time step $t \in \{1, \ldots, T\}$, the agent samples a parameter estimate $\hat{\Theta}_t$ from the posterior distribution, conditioned on the history $H^t$. The agent then selects the action that is optimal for the sampled parameter, $A_t = \pi_\star(\hat{\Theta}_t)$, receives a reward $R_t$, and updates the history to $H^{t+1} = H^t \cup \{\hat{A}_t, R_t\}$.
The pseudocode for Thompson Sampling is given in Algorithm \ref{alg:Thompson_Sampling}.

\begin{algorithm}[ht]
    \caption{Thompson Sampling algorithm}
    \label{alg:Thompson_Sampling}
    \begin{algorithmic}[1]
        \STATE {\bfseries Input:} parameter prior $\bP_{\Theta}$.
        \FOR{$t=1$ {\bfseries to} T}
            \STATE Sample a parameter estimate $\smash{\hat{\Theta}_t \sim \bP_{\Theta|H^t}}$.
            \STATE Take the corresponding optimal action $A_t = \pi_\star(\hat{\Theta}_t)$.
            \STATE Collect the reward $R_t = R(A_t, \Theta)$.
            \STATE Update the history $H^{t+1} = H^t \cup \{A_t, R_t\}$.
        \ENDFOR
    \end{algorithmic}
\end{algorithm}

\subsection{Notations Specific to Bandit Problems}
\label{subsec:notations_specific}

Since the $\sigma$-algebras of the history are often used in conditioning, we introduce the notations $\bE_t[\cdot] \coloneqq \bE[\cdot | H^t]$ and $\bP_t[\cdot] \coloneqq \bP[\cdot | H^t]$ to denote the conditional expectation and probability given $H^t$. Additionally, we define $\mi_t(A^\star; R_t|A_t) \coloneqq \bE_t[\kl(\bP_{R_t | H^t, A^\star,A_t} \| \bP_{R_t | H^t,A_t})]$ as the disintegrated conditional mutual information between the optimal action $A^\star$ and the reward $R_t$ conditioned on action $A_t$, \emph{given the history} $H^t$.

\section{Information ratio and Compressed TS}
\label{sec:information_ratio}

In their analysis of Thompson Sampling for bandit problems, Russo and Van Roy~\cite{russo_information-theoretic_2015} introduced an important quantity, the \emph{information ratio} defined as the following random variable:

\begin{equation*}
    \Gamma_t(A^\star,A_t) \coloneqq \frac{\bE_t[ R(A^\star,\Theta) - R(A_t,\Theta)]^2}{\mi_t(A^\star; R(A_t,\Theta),A_t)}.
\end{equation*}
 This ratio measures the trade-off between minimizing the current squared regret and gathering information about the optimal action. In other words, a small information ratio implies that any significant regret has been compensated by a substantial gain in information about the optimal action.\\
 
Russo and Van Roy use this concept to provide a general regret bound of $\sqrt{\Gamma \cdot T \cdot \ent(A^\star)}$, which depends on the time horizon $T$, the entropy of the prior distribution of $A^\star$, and an algorithm- and problem-dependent upper bound $\Gamma$ on the average expected information ratio~\cite[Proposition 1]{russo_information-theoretic_2015}. A limitation of this approach is that the prior entropy of the optimal action, $\ent(A^\star)$, can grow arbitrarily large with the number of actions or get infinite if the action space is continuous. Dong and Van Roy~\cite{dong_information-theoretic_2018} extended this analysis by controlling the regret of the Thompson Sampling through the regret of an approximate learning they refer to as \emph{one-step compressed} Thompson Sampling. Under a continuity assumption of the expected reward with respect to the action space, they upper bound the regret via a \emph{compressed statistic} $\Theta_\varepsilon$ of the parameter $\Theta$, along with the information ratio of the one-step compressed Thompson Sampling defined as  
\begin{equation*}
    \tilde{\Gamma}_t(\tilde{\Theta}_t^\star,\tilde{\Theta}_t) \coloneqq \frac{\bE_t[ R(\pi_\star(\tilde{\Theta}_t^\star),\Theta) - R(\pi_\star(\tilde{\Theta}_t),\Theta)]^2}{\mi_t(\tilde{\Theta}_t^\star; R(\pi_\star(\tilde{\Theta}_t),\Theta),\tilde{\Theta}_t)}.
\end{equation*} 
for some carefully crafted variables $\tilde{\Theta}_t^\star$ and $\tilde{\Theta}_t$ which depend respectively on $\Theta_\varepsilon$ and $\hat{\Theta}_{\varepsilon,t}$ (a statistic corresponding to the parameter estimate $\hat{\Theta}_t$).
The resulting bound is of the form $\varepsilon\cdot T + \sqrt{\tilde{\Gamma}\cdot T\cdot \ent(\Theta_\varepsilon)}$ where $\tilde{\Gamma}$ is a problem dependent upper bound on the average expected $\tilde{\Gamma}_t(\tilde{\Theta}_t^\star,\tilde{\Theta}_t)$ and $\varepsilon$ is a parameter that controls the quantization of $\Theta_\varepsilon$.
For $d$-dimensional linear bandit, they show $\tilde{\Gamma} \leq d/2$ in~\cite[Proposition 3]{dong_information-theoretic_2018} and after optimizing over the choice of $\varepsilon$ they derive a near-optimal regret rate of $O(d \sqrt{T \log T})$.\\ 

However, those results do not hold for infinite or continuous action and parameter spaces, as the proof techniques for constructing the one-step compressed Thompson Sampling variables $\tilde{\Theta}_t^\star$ and $\tilde{\Theta}_t$~\cite[Proposition 2]{dong_information-theoretic_2018} specifically require a finite parameter space or finite action support. We address this limitation in~\Cref{sec:main_results}, and, to simplify the exposure of the results, we adapt their construction such that it depends on a \emph{statistic} $A^\star_\varepsilon$ of the optimal action $A^\star$ instead of a statistic $\Theta_\varepsilon$ of the parameter $\Theta$.

\section{Main Results}
\label{sec:main_results}
In this section, we explain how we extend the results of Dong and Van Roy~\cite{dong_information-theoretic_2018} to continuous and infinite parameter spaces. We then apply the regret bound to linear bandits.  To simplify the expressions, we will use the notation $R(A_t)$ as a shorthand for $R(A_t,\Theta)$.

We begin by adapting their construction of a one-step compressed Thompson Sampling such that it depends on a statistic $A^\star_\varepsilon$ of the optimal action $A^\star$ and a corresponding carefully crafted action sampling function $\phi_t: \cA \to \cA$, such that  $\tilde{A}_t^\star \coloneqq \phi_t(A^\star_\varepsilon)$, for each round $t \in \{ 1, \ldots, T \}$. Similar to~\cite[Proposition 2]{dong_information-theoretic_2018}, this statistic and sampling functions are designed to simultaneously satisfy the following three requirements:
\begin{enumerate}[label=(\roman*)]
    \item \label{requirement:less_informative} The statistic $A^\star_\varepsilon$ is less informative than $A^\star$, that is, $\ent(A^\star_\varepsilon) \leq \ent(A^\star)$.
    \item \label{requirement:approx}At each round $t\in \{1,\ldots,T\}$, the one-step compressed Thompson Sampling regret is ``$\varepsilon$-close'' to the Thompson Sampling regret.
    \item \label{requirement:no_more_information}For each time step $t\in\{1,\ldots,T\}$, the one-step compressed Thompson Sampling regret can be bounded using the information gained about the statistic $A^\star_\varepsilon$. At the same time, it reveals no more information about $A^\star_\varepsilon$ than Thompson Sampling. 
\end{enumerate}

Following~\cite{dong_information-theoretic_2018}, we construct a partition of $\{\cA_k\}_{k=1}^K$ of $\cA$ such that for each $k=1,\ldots,K$, for all $a,a'\in \cA_k$, $\bE[R(a,\pi_\star^{-1}(a))-R(a',\pi_\star^{-1}(a))] \leq \varepsilon$. We define the statistic $A^\star_\varepsilon$ as the random variable recording the partition of $A^\star$: 
\begin{align}
\label{eq:def_statistic}
    A^\star_\varepsilon = k \iff A^\star \in \cA_k.
\end{align}

To prove the existence of a one-step compressed Thompson Sampling satisfying requirements \ref{requirement:less_informative}, \ref{requirement:approx} and \ref{requirement:no_more_information}, we introduce \Cref{prop:new_lemma}.
\begin{proposition}
\label{prop:new_lemma}
    Consider a space $\cA$, two functions $f: \cA \to \bR_+$ and $g: \cA \to \bR_+$, and a probability distribution $\bQ$ on $\cA$. Then, there exists a pair $(a_1, a_2) \in \cA^2$ and a $q \in [0,1]$ such that
    \begin{equation*}
        q f(a_1) + (1-q) f(a_2) \leq \int_{a\in \cA} f(a) \mathrm{d} \bQ(a) 
    \end{equation*}
    and 
    \begin{equation*}
        q g(a_1) + (1-q) g(a_2) \leq \int_{a\in \cA} g(a) \mathrm{d} \bQ(a).
    \end{equation*}
\end{proposition}

\begin{proof}
    Let $\bar{F} = \int_{a \in \cA}  f(a) \mathrm{d} \bQ(a)$ and $\bar{G} = \int_{a \in \cA} g(a) \mathrm{d} \bQ(a)$. Now, consider the spaces $\cA_f \coloneqq \{ a \in \cA: f(a) \leq \bar{F} \}$ and $\cA_g \coloneqq \{ a \in \cA : g(a) \leq \bar{G} \}$. If $\cA_f \cap \cA_g \neq \emptyset$, then taking both $a_1$ and $a_2$ from $\cA_f \cap \cA_g$  trivially satisfies the conditions for all $q \in [0,1]$. Therefore, let us assume that the sets are disjoint for the rest of the proof.

    Consider some $a_1 \in \cA_f = \cA_g^c$ and some $a_2 \in \cA_g = \cA_f^c$. We can rewrite the required condition from the lemma as 
    \begin{equation*}
        q \geq \frac{f(a_2) - \bar{F}}{f(a_2) - f(a_1)} \quad \textnormal{ and } \quad q \leq \frac{\bar{G} - g(a_2)}{g(a_1) - g(a_2)},
    \end{equation*}
    where the first inequality took into account that $f(a_1) < f(a_2)$ by the definition of the sets $\cA_f$ and $\cA_g = \cA_f^c$. This inequality can, in turn, be written as
    \begin{equation*}
        \frac{f(a_2) - \bar{F}}{f(a_2) - f(a_1)} \leq \frac{\bar{G} - g(a_2)}{g(a_1) - g(a_2)} 
    \end{equation*}
    which is equivalent to
    \begin{equation*}
        f(a_2) g(a_1) \smash{-}\bar{F} \big( g(a_1)\smash{-}g(a_2) \big)\smash{ \leq }\bar{G} \big( f(a_2)\smash{-} f(a_1) \big) \smash{+}f(a_1) g(a_2).
    \end{equation*}

    At this point, we have all the ingredients to prove the statement by contradiction. Assume that there is no pair $(a_1, a_2) \in \cA_f \times \cA_g$ such that the condition holds, then 
    \begin{equation*}
    \label{eq:contradiction_equation}
        f(a_2) g(a_1)\smash{-} \bar{F} \big( g(a_1)\smash{-} g(a_2) \big)\smash{>} \bar{G} \big( f(a_2)\smash{-} f(a_1) \big)\smash{+} f(a_1) g(a_2)
    \end{equation*}
    for every pair $(a_1, a_2) \in \cA_f \times \cA_g$. Therefore, we can integrate over all such pairs, and the inequality should still hold, namely
    \begin{align*}
        &\int_{\cA_f} \int_{\cA_g} \bigg[ f(a_2) g(a_1)\smash{-} \bar{F} \big( g(a_1)\smash{-} g(a_2) \big) \bigg] \mathrm{d} \bQ(a_1) \mathrm{d} \bQ(a_2) \nonumber \\ 
        &\smash{>} \int_{\cA_f} \int_{\cA_g} \bigg[ \bar{G} \big( f(a_2)\smash{-} f(a_1) \big)\smash{+}f(a_1) g(a_2) \bigg] \mathrm{d} \bQ(a_1) \mathrm{d} \bQ(a_2).
    \end{align*}
    We must introduce some notation to show that the above inequality cannot happen. Let $F^- \coloneqq \int_{\cA_f} f(a) \mathrm{d} \bQ(a)$ and $F^+ \coloneqq \int_{\cA_g} f(a) \mathrm{d} \bQ(a)$ and note that $F^+ + F^- = \bar{F}$. Similarly, $G^- \coloneqq \int_{\cA_g} g(a) \mathrm{d} \bQ(a)$ and $G^+ \coloneqq \int_{\cA_f} g(a) \mathrm{d} \bQ(a)$ and $G^+ +G^- = \bar{G}$. Using this notation, we can use Fubini's theorem and rewrite the above inequality as
    \begin{equation*}
        F^+G^+\smash{-}(F^+\smash{+}F^-)(G^+\smash{-}G^-)\smash{>}(G^+\smash{+} G^-) (F^+\smash{-}F^-)\smash{+}F^- G^-
    \end{equation*}
    which can be simplified to $F^- G^- > F^+ G^+$
    and which is impossible by the definition of $F^-$, $F^+$, $G^+$ and $G^-$, completing the contradiction and therefore the proof. 
\end{proof}

Equipped with \Cref{prop:new_lemma}, we can now extend~\cite[Proposition 2]{dong_information-theoretic_2018} to infinite and continuous action space. 

\begin{restatable}{proposition}{ExistenceApproximateLearning}
\label{prop:existence_learning}
Let $A^\star_\varepsilon$ be defined as in~\eqref{eq:def_statistic}. For each time step $t \in \{1,\ldots,T\}$, there exists a of random functions $\phi_t$ that satisfies the following:
\begin{enumerate}
\item 
\label{prop:existence_regret_bound}
$\bE_t[R(A^\star)- R(A_t)] \leq \varepsilon +\bE_t[R(\tilde{A}^\star_t)-R(\tilde{A}_{t,\varepsilon})].$
\item 
\label{prop:existence_mutual_info_bound}$\mi_t\left(A^\star_\varepsilon;\tilde{A}_{t,\varepsilon},R(\tilde{A}_{t,\varepsilon})\right) \leq \mi_t\left(A^\star_\varepsilon;\hat{A}_t,R(\hat{A}_t)\right). $
\end{enumerate}
where we set $\tilde{A}^\star_t = \phi_t(A^\star_\varepsilon)$ and $\tilde{A}_{t,\varepsilon} = \phi_t(\hat{A}_{t,\varepsilon})$,  and where $\hat{A}_{t,\varepsilon}$ is the statistic corresponding to $\hat{A}_t$, thus $A^\star_\varepsilon$ and $\hat{A}_{t,\varepsilon}$ are identically distributed conditioned on $H^t$.
\end{restatable}
\begin{proof}
The proof follows the technique from~\cite[Proof of Proposition 2]{dong_information-theoretic_2018} closely with the key difference that we use~\Cref{prop:new_lemma} instead of~\cite[Lemma 1]{dong_information-theoretic_2018}.
\end{proof}

Adjusting to our construction of \emph{one-step compressed TS} based on $A^\star_\varepsilon$, we adapt the definition of information ratio as
\begin{align*}
     \tilde{\Gamma}_{t}(\tilde{A}^\star_t,\tilde{A}_{t,\varepsilon}) \coloneqq \frac{\bE_t[ R(\tilde{A}^\star_t,\Theta) - R(\tilde{A}_{t,\varepsilon},\Theta)]^2}{\mi_t\left(A^\star_\varepsilon;\tilde{A}_{t,\varepsilon},R(\tilde{A}_{t,\varepsilon})\right)}.
\end{align*}
We can now state and prove our main theorem. 
\begin{theorem}
\label{thm:main_theorem}
    Let $A^\star_\varepsilon$ be defined as in~\eqref{eq:def_statistic} and let $\tilde{A}^\star_t$ and $\tilde{A}_{t,\varepsilon}$ satisfy~\Cref{prop:new_lemma}. Assume that the average expected one-step compressed TS information ratio is bounded, $\frac{1}{T} \sum_{t=1}^T \bE[\tilde{\Gamma}_{t}(\tilde{A}^\star_t,\tilde{A}_{t,\varepsilon})] \leq \tilde{\Gamma}$, for some $\tilde{\Gamma} > 0$. Then, the TS cumulative regret is bounded as
    \begin{equation*}
        \bE[\textnormal{Regret}(T)] \leq  \sqrt{\tilde{\Gamma} T \ent(A^\star_\varepsilon) } +\varepsilon T.
    \end{equation*}
\end{theorem}
\begin{proof}
The proof follows the techniques of~\cite[Proof of Theorem 1]{dong_information-theoretic_2018}. We start the proof by upper bounding the regret of TS by the regret of the one-step compressed TS using the first inequality in \Cref{prop:existence_learning}: 
\begin{align}
    \bE[\textnormal{Regret}(T)] \leq \varepsilon T + \sum_{t=1}^T \bE[R(\tilde{A}^\star_t)-R(\tilde{A}_{t,\varepsilon})].
    \label{eq:regret_approx}
\end{align}
We rewrite the expected regret of one-step compressed TS using the definition of the information ratio: 
\begin{align*}
    \bE[R(\tilde{A}^\star_t)\smash{-}R(&\tilde{A}_{t,\varepsilon})]\smash{=}  \bE\left[\sqrt{\tilde{\Gamma}_{t}(\tilde{A}^\star_t,\tilde{A}_{t,\varepsilon}) \mi_t(A^\star_\varepsilon;\tilde{A}_{t,\varepsilon},R(\tilde{A}_{t,\varepsilon}))}\right]
\end{align*}
We continue using Jensen's inequality and applying Cauchy-Schwartz inequality:
\begin{align*}
    \eqref{eq:regret_approx}
    &\leq \varepsilon T +\sqrt{\tilde{\Gamma} T \sum_{t=1}^T \mi(A^\star_\varepsilon;\tilde{A}_{t,\varepsilon},R(\tilde{A}_{t,\varepsilon})|H^t)}
\end{align*}
where in the last inequality, we used the assumption that $\sum_{t=1}^T \bE_t[\tilde{\Gamma}_{t}(\tilde{A}^\star_t,\tilde{A}_{t,\varepsilon})] \leq \tilde{\Gamma} T$. Using the second inequality in~\Cref{prop:existence_learning}, summing the $T$ mutual information $\mi(\Theta_\varepsilon; R_t |H^t,  A_t)$ and applying the chain rule, we obtain
\begin{align*}
    \bE[\textnormal{Regret}(T)] \leq \varepsilon T + \sqrt{\tilde{\Gamma} T \mi(A^\star_\varepsilon;H^T)}.
\end{align*}
Finally, we upper bound the mutual information $\mi(A^\star_\varepsilon;H^T)$ by the entropy $\ent(A^\star_\varepsilon)$ to obtain the claimed result. 
\end{proof}
For bandit problems with expected rewards that are $L$-Lipschitz with respect to a metric action space $(\cA,\rho)$, we can derive a regret bound from~\Cref{thm:main_theorem} that depends on a measure of the complexity of the action space. 
\begin{corollary}
\label{corrollary:main}
    Assume that $\bE[R(a,\theta)]$ is $L$-Lipschitz with respect to the action space $(\cA,\rho)$ for all $\theta \in \cO$, and assume that the average expected one-step compressed TS information ratio is bounded, $\frac{1}{T} \sum_{t=1}^T \bE[\tilde{\Gamma}_{t}(\tilde{A}^\star_t,\tilde{A}_{t,\varepsilon})] \leq \tilde{\Gamma}$, for some $\tilde{\Gamma} > 0$. Then, the TS cumulative regret is bounded as
    \begin{equation*}
        \bE[\textnormal{Regret}(T)] \leq  \sqrt{\tilde{\Gamma} T \log(\cN(\cA,\rho,\epsilon)) } +L\epsilon T,
    \end{equation*}
    where $\cN(\cA,\rho,\epsilon)$ is the covering number for the metric space $(\cA,\rho)$ at scale $\epsilon\geq0$.
\end{corollary}
\begin{proof}
    Under the Lipschitz continuity property of the expected rewards, we note that setting $A^\star_\varepsilon$ as the quantization of $A^\star$ on a $\varepsilon/L$-net for $(\cA,\rho)$ satisfies~\eqref{eq:def_statistic}.
    Starting from~\Cref{thm:main_theorem}, we upper bound the entropy $\ent(A^\star_\varepsilon)$ by the logarithm of the cardinality of $A^\star_\varepsilon$. Choosing the $\varepsilon/L$-net with the smallest cardinality and setting $\epsilon=\varepsilon/L$ yields the claimed result.
\end{proof}

\section{Application to linear bandit problems}

In \emph{linear bandits} problems, both the actions and the environment parameter are parameterized by a feature vector, and the associated expected reward can be written as their inner product. Mathematically, a $d$-dimensional linear bandit problem is a bandit problem with $\cA,\cO \subset \bR^d$ and such that for all $a\in \cA$ and all $\theta \in \cO$ we have \begin{align*}
    \bE[R(a,\theta)] = \innerproduct{a}{\theta},
\end{align*}

where the expectation is taken over the reward distribution.\\

Adapting~\Cref{corrollary:main} to linear bandits, we get the following.
\begin{proposition}
\label{prop:linear_bandits}
    For $d$-linear bandit problems with rewards bounded in $[-1,1]$, the TS cumulative regret is bounded as
    \begin{align*}
        \bE[\textnormal{Regret}(T)] \leq \sqrt{2 d T \log(\cN(\cA,||\cdot||_2,\varepsilon)) } +\varepsilon T
    \end{align*}
\end{proposition}
\begin{proof}
   From~\cite[Proposition 3]{dong_information-theoretic_2018}, we have that the average expected one-step compressed TS information ratio is bounded by $2d$. Then, noting that linear bandits are $1$-Lipschitz with respect to $(\cA,||\cdot||_2)$ where $||\cdot||_2$ is the Euclidean distance, and applying~\Cref{corrollary:main} gives the claimed result. 
\end{proof}
Provided that the diameter of the action space $\cA$ is bounded, we can control the covering number $\cN(\cA,||\cdot||_2,\varepsilon)$ and get a regret bound depending only on the dimension $d$ and the time horizon $T$. It extends~\cite[Theorem 2]{dong_information-theoretic_2018} to general action spaces and improves on the constant factors inside the logarithm\footnote{In~\cite[Theorem 2]{dong_information-theoretic_2018}, the expected rewards are $\nicefrac{1}{2}$-Lipschitz and the rewards in $[\nicefrac{\smash{-}1}{2},\nicefrac{1}{2}]$. With the same assumptions, our bound reduces by a factor $2$.  }.  
\begin{corollary}
    For $d$-linear bandit problems with rewards bounded in $[-1,1]$, suppose that $\cA \subseteq \mathbf{B}_d(0,1)$, where $\mathbf{B}_d(0,1)$ is the $d$-dimensional closed Euclidean unit ball. 
    Then, the TS cumulative regret is bounded as
    \begin{equation*}
        \bE[\textnormal{Regret}(T)] \leq  2d\sqrt{T \log\left(\sqrt{2}+ \frac{4\sqrt{T}}{d}\right)}.
    \end{equation*}
\end{corollary}
\begin{proof}
    The proof closely follows the technique from~\cite[Proof of Theorem 2]{dong_information-theoretic_2018}. Starting from~\Cref{prop:linear_bandits}, since $\cA \subseteq \mathbf{B}_d(0,1)$ and $\rho$ is the Euclidean distance, the logarithm of the covering number is bounded by $d \log(1+\nicefrac{2}{\varepsilon})$(see~\cite[Proof of lemma 5.13]{van_handel_probability_2016}). By setting $\varepsilon=\frac{d}{2\sqrt{T}}$ and using properties of square roots and of logarithms, we obtain the claimed result.
\end{proof}

\section{Conclusion}
In this paper, we studied the Bayesian regret of the Thompson Sampling algorithm for bandit problems. Building on the analysis from Dong and Van Roy~\cite{dong_information-theoretic_2018}, we established bounds on the TS expected regret that hold even for problems with infinite and continuous action and parameter spaces. For the linear bandit problem, our analysis recovers the near-optimal rate of $O(d \sqrt{T \log T})$ from~\cite[Theorem 2]{dong_information-theoretic_2018} and improves the constants terms inside the logarithm. A natural direction for future work is to extend our results to other classes of bandit problems, such as the ``semi-bandit'' feedback problem.

\clearpage
\bibliographystyle{IEEEtran}
\balance
\bibliography{references}

\begin{thebibliography}{10}
\providecommand{\url}[1]{#1}
\csname url@samestyle\endcsname
\providecommand{\newblock}{\relax}
\providecommand{\bibinfo}[2]{#2}
\providecommand{\BIBentrySTDinterwordspacing}{\spaceskip=0pt\relax}
\providecommand{\BIBentryALTinterwordstretchfactor}{4}
\providecommand{\BIBentryALTinterwordspacing}{\spaceskip=\fontdimen2\font plus
\BIBentryALTinterwordstretchfactor\fontdimen3\font minus \fontdimen4\font\relax}
\providecommand{\BIBforeignlanguage}[2]{{%
\expandafter\ifx\csname l@#1\endcsname\relax
\typeout{** WARNING: IEEEtran.bst: No hyphenation pattern has been}%
\typeout{** loaded for the language `#1'. Using the pattern for}%
\typeout{** the default language instead.}%
\else
\language=\csname l@#1\endcsname
\fi
#2}}
\providecommand{\BIBdecl}{\relax}
\BIBdecl

\bibitem{russo_information-theoretic_2015}
\BIBentryALTinterwordspacing
D.~Russo and B.~Van~Roy, ``\BIBforeignlanguage{en}{An {Information}-{Theoretic} {Analysis} of {Thompson} {Sampling}},'' Jun. 2015, number: arXiv:1403.5341 arXiv:1403.5341 [cs]. [Online]. Available: \url{http://arxiv.org/abs/1403.5341}
\BIBentrySTDinterwordspacing

\bibitem{dong_information-theoretic_2018}
\BIBentryALTinterwordspacing
S.~Dong and B.~Van~Roy, ``\BIBforeignlanguage{en}{An {Information}-{Theoretic} {Analysis} for {Thompson} {Sampling} with {Many} {Actions}},'' May 2018, arXiv:1805.11845 [cs, math, stat]. [Online]. Available: \url{http://arxiv.org/abs/1805.11845}
\BIBentrySTDinterwordspacing

\bibitem{bouneffouf_survey_2020}
\BIBentryALTinterwordspacing
D.~Bouneffouf, I.~Rish, and C.~Aggarwal, ``\BIBforeignlanguage{en}{Survey on {Applications} of {Multi}-{Armed} and {Contextual} {Bandits}},'' in \emph{\BIBforeignlanguage{en}{2020 {IEEE} {Congress} on {Evolutionary} {Computation} ({CEC})}}.\hskip 1em plus 0.5em minus 0.4em\relax Glasgow, United Kingdom: IEEE, Jul. 2020, pp. 1--8. [Online]. Available: \url{https://ieeexplore.ieee.org/document/9185782/}
\BIBentrySTDinterwordspacing

\bibitem{silva_multi-armed_2022}
\BIBentryALTinterwordspacing
N.~Silva, H.~Werneck, T.~Silva, A.~C. Pereira, and L.~Rocha, ``\BIBforeignlanguage{en}{Multi-{Armed} {Bandits} in {Recommendation} {Systems}: {A} survey of the state-of-the-art and future directions},'' \emph{\BIBforeignlanguage{en}{Expert Systems with Applications}}, vol. 197, p. 116669, Jul. 2022. [Online]. Available: \url{https://linkinghub.elsevier.com/retrieve/pii/S0957417422001543}
\BIBentrySTDinterwordspacing

\bibitem{abe_reinforcement_2003}
\BIBentryALTinterwordspacing
N.~Abe, A.~W. Biermann, and P.~M. Long, ``\BIBforeignlanguage{en}{Reinforcement {Learning} with {Immediate} {Rewards} and {Linear} {Hypotheses}},'' \emph{\BIBforeignlanguage{en}{Algorithmica}}, vol.~37, no.~4, pp. 263--293, Dec. 2003. [Online]. Available: \url{http://link.springer.com/10.1007/s00453-003-1038-1}
\BIBentrySTDinterwordspacing

\bibitem{russo_tutorial_2018}
\BIBentryALTinterwordspacing
D.~J. Russo, B.~V. Roy, A.~Kazerouni, I.~Osband, and Z.~Wen, ``\BIBforeignlanguage{English}{A {Tutorial} on {Thompson} {Sampling}},'' \emph{\BIBforeignlanguage{English}{Foundations and Trends® in Machine Learning}}, vol.~11, no.~1, pp. 1--96, Jul. 2018, publisher: Now Publishers, Inc. [Online]. Available: \url{https://www.nowpublishers.com/article/Details/MAL-070}
\BIBentrySTDinterwordspacing

\bibitem{chu_contextual_2011}
W.~Chu, L.~Li, L.~Reyzin, and R.~E. Schapire, ``\BIBforeignlanguage{en}{Contextual {Bandits} with {Linear} {Payoﬀ} {Functions}},'' \emph{\BIBforeignlanguage{en}{JMLR Workshop and Conference Proceedings}}, 2011.

\bibitem{agrawal_optimistic_2023}
\BIBentryALTinterwordspacing
S.~Agrawal and R.~Jia, ``\BIBforeignlanguage{en}{Optimistic {Posterior} {Sampling} for {Reinforcement} {Learning}: {Worst}-{Case} {Regret} {Bounds}},'' \emph{\BIBforeignlanguage{en}{Mathematics of Operations Research}}, vol.~48, no.~1, pp. 363--392, Feb. 2023. [Online]. Available: \url{https://pubsonline.informs.org/doi/10.1287/moor.2022.1266}
\BIBentrySTDinterwordspacing

\bibitem{foster_contextual_2018}
\BIBentryALTinterwordspacing
D.~J. Foster and A.~Krishnamurthy, ``\BIBforeignlanguage{en}{Contextual bandits with surrogate losses: {Margin} bounds and efficient algorithms},'' Nov. 2018, arXiv:1806.10745 [cs, stat]. [Online]. Available: \url{http://arxiv.org/abs/1806.10745}
\BIBentrySTDinterwordspacing

\bibitem{foster_efficient_2021}
------, ``\BIBforeignlanguage{en}{Efficient first-order contextual bandits: {Prediction}, allocation, and triangular discrimination.}'' \emph{\BIBforeignlanguage{en}{Advances in Neural Information Processing Systems}}, 2021.

\bibitem{zhang_feel-good_2021}
\BIBentryALTinterwordspacing
T.~Zhang, ``\BIBforeignlanguage{en}{Feel-{Good} {Thompson} {Sampling} for {Contextual} {Bandits} and {Reinforcement} {Learning}},'' Oct. 2021, arXiv:2110.00871 [cs, math, stat]. [Online]. Available: \url{http://arxiv.org/abs/2110.00871}
\BIBentrySTDinterwordspacing

\bibitem{neu_lifting_2022}
\BIBentryALTinterwordspacing
G.~Neu, I.~Olkhovskaia, M.~Papini, and L.~Schwartz, ``\BIBforeignlanguage{en}{Lifting the {Information} {Ratio}: {An} {Information}-{Theoretic} {Analysis} of {Thompson} {Sampling} for {Contextual} {Bandits}},'' \emph{\BIBforeignlanguage{en}{Advances in Neural Information Processing Systems}}, vol.~35, pp. 9486--9498, Dec. 2022. [Online]. Available: \url{https://proceedings.neurips.cc/paper_files/paper/2022/hash/3d84d9b523e6e82916d496e58761002e-Abstract-Conference.html}
\BIBentrySTDinterwordspacing

\bibitem{thompson_likelihood_1933}
\BIBentryALTinterwordspacing
W.~R. Thompson, ``On the likelihood that one unknown probability exceeds another in view of the evidence of two samples,'' \emph{Biometrika}, vol.~25, no. 3-4, pp. 285--294, Dec. 1933. [Online]. Available: \url{https://doi.org/10.1093/biomet/25.3-4.285}
\BIBentrySTDinterwordspacing

\bibitem{russo_learning_2017}
\BIBentryALTinterwordspacing
D.~Russo and B.~Van~Roy, ``\BIBforeignlanguage{en}{Learning to {Optimize} via {Information}-{Directed} {Sampling}},'' Jul. 2017, arXiv:1403.5556 [cs]. [Online]. Available: \url{http://arxiv.org/abs/1403.5556}
\BIBentrySTDinterwordspacing

\bibitem{chapelle_empirical_2011}
\BIBentryALTinterwordspacing
O.~Chapelle and L.~Li, ``An {Empirical} {Evaluation} of {Thompson} {Sampling},'' in \emph{Advances in {Neural} {Information} {Processing} {Systems}}, vol.~24.\hskip 1em plus 0.5em minus 0.4em\relax Curran Associates, Inc., 2011. [Online]. Available: \url{https://proceedings.neurips.cc/paper/2011/hash/e53a0a2978c28872a4505bdb51db06dc-Abstract.html}
\BIBentrySTDinterwordspacing

\bibitem{dong_performance_2019}
S.~Dong, T.~Ma, and B.~V. Roy, ``\BIBforeignlanguage{en}{On the {Performance} of {Thompson} {Sampling} on {Logistic} {Bandits}},'' \emph{\BIBforeignlanguage{en}{Conference on Learning Theory}}, 2019.

\bibitem{van_handel_probability_2016}
R.~van Handel, \emph{\BIBforeignlanguage{en}{Probability in {High} {Dimension}}}.\hskip 1em plus 0.5em minus 0.4em\relax Princeton University, Dec. 2016, vol. APC 550 Lecture Notes.

\end{thebibliography}

\end{document}